\def\year{2021}\relax
\documentclass[letterpaper]{article} 
\usepackage{aaai21}  
\usepackage{times}  
\usepackage{helvet} 
\usepackage{courier}  
\usepackage[hyphens]{url}  
\usepackage{graphicx} 
\usepackage{natbib}  
\usepackage{caption} 
\usepackage{booktabs} 
\usepackage{amsfonts}
\usepackage{nicefrac}
\usepackage{microtype}
\usepackage{todonotes}
\usepackage{enumitem}
\usepackage{amsmath}
\usepackage{amsthm}
\usepackage{subcaption}
\usepackage{multicol}
\usepackage{multirow}
\DeclareMathOperator*{\argmax}{arg\,max}

\usepackage[ruled,vlined]{algorithm2e}
\usepackage[switch]{lineno}

\newtheorem{theorem}{Theorem}
\newcommand{\tpm}{$\pm$}

\title{Joint Text and Label Generation for Spoken Language Understanding}
\author {
        Yang Li,\textsuperscript{\rm 1}\footnote{work done while interning at Amazon}
        Ben Athiwaratkun, \textsuperscript{\rm 2}
        Cicero Nogueira dos Santos, \textsuperscript{\rm 2}
        Bing Xiang, \textsuperscript{\rm 2}\\
}
\affiliations {
    \textsuperscript{\rm 1} University of North Carolina at Chapel Hill \\
    \textsuperscript{\rm 2} Amazon Web Services \\
    yangli95@cs.unc.edu, \{benathi,cicnog,bxiang\}@amazon.com
}

\begin{document}

\maketitle

\begin{abstract}
Generalization is a central problem in machine learning, especially when data is limited. Using prior information to enforce constraints is the principled way of encouraging generalization. In this work, we propose to leverage the prior information embedded in pretrained language models (LM) to improve generalization for intent classification and slot labeling tasks with limited training data. Specifically, we extract prior knowledge from pretrained LM in the form of synthetic data, which encode the prior implicitly. 
We fine-tune the LM to generate an \emph{augmented language}, which contains not only text but also encodes both intent labels and slot labels. 
The generated synthetic data can be used to train a classifier later.
Since the generated data may contain noise, we rephrase the learning from generated data as learning with noisy labels. We then utilize the mixout regularization for the classifier and prove its effectiveness to resist label noise in generated data. Empirically, our method demonstrates superior performance and outperforms the baseline by a large margin.
\end{abstract}

\section{Introduction}

\begin{figure*}
    \centering
    \includegraphics[width=.98\linewidth]{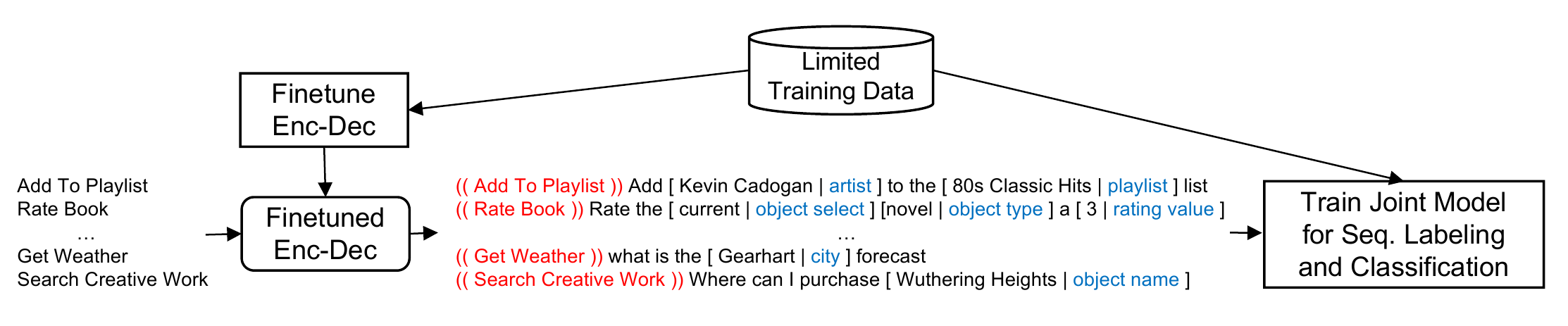}
    \caption{Proposed pipeline for spoken language understanding via data augmentation with joint generation of text and labels.}
    \label{fig:pipeline}
\end{figure*}

Natural language processing has been profoundly impacted by leveraging pretrained large-scale language models. Many downstream tasks have achieved state-of-the-art performance by fine-tuning pretrained models \cite{devlin2018bert,radford2019language,raffel2019exploring}. The success lies in the transfer ability of those models trained with large unlabeled corpus, which are typically thought of learning universal language representations \cite{howard2018universal}. Although the wide adoption, fine-tuning still requires a large dataset to achieve good performance. In some cases, however, it would be inconvenient, expensive, or even impossible to collect a large dataset. For instance, when adapting a personal voice assistant to a specific user, it is inconvenient to require the user to label a lot of utterances.

Learning with a few examples, or few-shot learning, is an active research area recently. However, classic few-shot mechanism requires to accumulate prior knowledge by training over a large set of related tasks, which essentially delegates the burden of collecting a large training dataset to collecting a large set of training tasks. More detrimentally, the distribution shift between training and testing tasks caused by different target classes can potentially lead to poor generalization \cite{guan2020few}. 

In this work, we focus on a different setting, where we only have access to very limited data from each testing domains. Learning with limited data is challenging since the modern over-parametrized models can easily overfit the small training dataset while simple models usually suffer from insufficient representative power. According to the Bayesian Occam's razor theory \cite{mackay1992bayesian}, exploiting prior information is a principled way of encouraging generalization when faced with limited data. In this work, we leverage the priors embedded in pretrained language models.

Using language model as a prior is not new. In statistical machine translation literature, language models have been used as priors for decades \cite{brown-etal-1993-mathematics}. Neural machine translation recently have also attempted to use language models as priors \cite{baziotis2020language}. Generally speaking, fine-tuning a pretrained language model can be regarded as updating the empirical model prior to a task-specific posterior \cite{kochurov2018bayesian}. For generative tasks, like machine translation, priors can be integrated by directly fine-tuning \cite{ramachandran2016unsupervised,skorokhodov2018semi} or constraining the output distribution to have high probability under prior language models \cite{baziotis2020language}. However, integrating priors for discriminative tasks is not straight-forward since they typically require to add additional classification layers on top of the backbone. It is difficult to find proper priors for those classification heads. Moreover, training from scratch could lead to overfitting.

Following the Bayesian view of data augmentation \cite{zhang2016understanding,dao2019kernel}, we express synthetic data as an implicit form of prior that encodes data and task invariances. In our proposed approach, 
prior information is distilled by generating task-specific synthetic data from pretrained language models. In order to generate task-specific data, we fine-tune the language models over the small training dataset. The augmented datasets are then used to train the classifiers. The synthetic data embody the prior by teaching the classifier about possible utterances for each label. The generation process also bears similarity to knowledge distillation \cite{hinton2015distilling}. Here, we distill prior knowledge from pretrained language models.

We focus on the tasks of intent classification and slot labeling, which correspond to sentence and token level classification tasks, respectively. Therefore, we require the synthetic data to contain both intent and slot labels. Existing conditional generation frameworks \cite{anaby2020not,yoo2019data} for data augmentation typically generate new sentences given a specific sentence level label, which cannot generate slot labels. A naive approach that condition on both intent labels and slot labels will not work, since there could be a combinatorial number of possible slot label sequences. Moreover, most of the combinations are invalid since they might never appear in real spoken languages. Another way of generating both sentences and labels is to factorize it using chain rule and introduce a labeling model to obtain labels for generated sentences. However, this is oxymoron since a good labeling model is exactly our ultimate goal.

To generate text (utterances) and labels simultaneously, we leverage an augmented language format \cite{anonymous}, 
where intent and slot label information is embedded in the generated sentences using an special format (see Fig.~\ref{fig:aug_lan}).
We fine-tune conditional language models to directly generate the augmented sentences. We test and compare multiple conditional generation strategies for synthesizing data with legitimate intent and slot labels.

Since we only have access to a very limited data to fine-tune the language model, it is inevitable that the generated data would contain noise. To resist label noise, we apply a recently proposed regularization method called mixout \cite{lee2019mixout}. Mixout is originally proposed to improve the generalization of large-scale language models. Here, we prove and empirically show that mixout regularized models are robust to label noise.

Our contributions are as follows:
\begin{itemize}
\item We leverage prior information embedded in pretrained language models to improve generalization of classification models. Specifically, we extract prior knowledge in the form of synthetic data and augment the small datasets with synthesized data and labels.
\item We utilize an augmented language format to simultaneously generate sentences and the corresponding intent and slot labels using multiple different conditional generation strategies. We empirically demonstrate the effectiveness of the augmented language format for data augmentation, especially with slot labels.
\item We propose two metrics that measure the correspondence between generated tokens and labels. Empirically, we show better correlation between the proposed metrics and the downstream task performance.
\item We reinterpret mixout as a regularization that resist label noise and prove its generalization bound under mild assumptions. We show that mixout regularized models guarantee the generalization on clean data.
\item We significantly outperform BERT-based baselines for joint slot labeling and intent classification in limited data regime. On SNIPS dataset, our proposed method improves the baseline for the slot labeling task by more than 30\% in F1.
\end{itemize}

\section{Methods}
In this section, we formally describe the problem and introduce our proposed method, which employs three main steps as illustrated in Fig \ref{fig:pipeline}: 
finetuning of a pretrained encoder-decoder (enc-dec) language model (our \emph{conditional generator}) using the small labeled set;
generation of synthetic data;
and training of the \emph{joint classifier} using the augmented training set.
We detail our conditional generative model and describe the modifications we made to deal with noisy generations. Finally, we introduce the evaluation metrics to compare generations and present our training procedure.

\begin{figure*}[t!]
    \centering
    \includegraphics[width=0.9\linewidth]{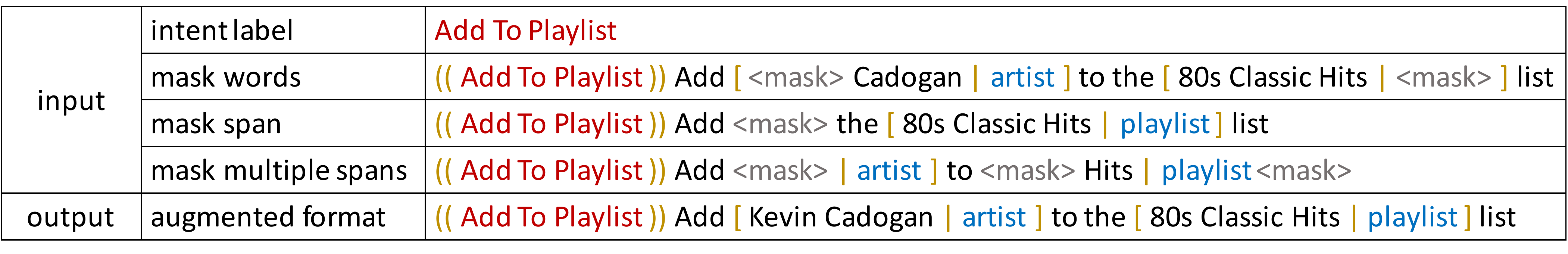}
    \caption{Input and output format of the conditional generator.}
    \label{fig:generator_io}
\end{figure*}

\subsection{Problem Formulation}
In this work, we are interested in slot labeling and intent classification. We focus on the limited data regime where a very small training dataset $\{(x_i,s_i,y_i)\}_{i=1}^{N}$ is given. $x$, $s$ and $y$ represent sentence, slot labels and intent label respectively. We will build a model $f(x;\mathbf{\theta}) = p_\mathbf{\theta}(s,y \mid x)$ to jointly classify slots and intents for novel sentences $x$:
\begin{equation*}
    s^*, y^* = \argmax_{s,y} p_\mathbf{\theta}(s, y \mid x)
\end{equation*}
We augment the small training set with synthetic data $\{(x',s',y')\}_{j=1}^{N'}$ generated from pretrained language models and train a joint classifier to predict intent and slot labels simultaneously. We will describe the generative and discriminative components below.

\subsection{Conditional Generator}
In this section, we introduce the generative component for drawing new training instances. First, we describe the augmented language format \cite{anonymous} utilized to generate utterances (text) and labels simultaneously. As shown in the last row in Fig.~\ref{fig:generator_io}, we use additional markers to indicate the token-spans and their associated labels. The augmented format can be converted from and to the traditional BIO format without loss of information (see Fig.~\ref{fig:aug_lan} in Appendix). With the training data converted into the augmented format, we can train a generative model to capture the joint distribution of the utterances and the labels, i.e., $p(x,s,y)$.

There are many options for modeling the joint distributions, such as VAE-based models \cite{kingma2013auto,yang2017improved}, autoregressive models \cite{radford2019language} and GANs \cite{goodfellow2014generative,de2019training}. In this work, we employ a seq2seq model \cite{raffel2019exploring} due to its verified ability for language modeling. It also gives us the flexibility to condition on additional information. That is, we essentially model the conditional distributions $p(x,s,y \mid c)$, where $c$ represents the conditioning information described later. The conditional generation mechanism enables us to control the generation quality and diversity by varying the conditioning input $c$.

We explore multiple different types of conditioning: 1) condition on the intent labels; 2) mask out several words from the augmented sentence at random; 3) mask out a random span from the augmented sentence; 4) mask out multiple spans. Please refer to Fig.~\ref{fig:generator_io} for an illustration.

Conditioning on intent label is simple and has been applied in many literature \cite{anaby2020not}. However, it cannot specify the possible slot labels for given intent labels. It leaves the burden of learning the correspondence between tokens and slot labels entirely on the model. Naively adding slot labels to the conditioning input will not work, since there are combinatorial number of possible slot label sequences and not all of them are valid combinations. 

Using masking scheme has been an effective way of learning language models \cite{devlin2018bert}. Here, we condition on a masked version of the augmented sentences so that the generator can partially leverage the correspondence between tokens and slot labels in the conditioning inputs. This format of conditioning is similar to rephrase a template where some tokens are replaced by synonyms extracted from a knowledge base \cite{wei2019eda}. Here, the masked tokens are replaced by other tokens extracted from the language models.

\subsubsection{Sampling and Filtering}
Sampling from the conditional generator is straight-forward, we just need to condition on the corresponding inputs. Intent labels are sampled from all possible labels and transformed to natural words separated by white spaces. Masked conditioning inputs are constructed by randomly sampling one augmented sentence from the given training set and replacing some tokens or spans with a mask token. The generated augmented sentences can then be decoded into BIO format and used to train the downstream models. However, generations could be noisy, some of them might not follow the exact format of the augmented language, thus cannot be decoded into the corresponding BIO format; some generations might contain invalid labels. We postprocess the generations by simply dropping those generations.

\subsection{Joint Classifier}
After generating the synthetic data, we train a classifier to jointly predict intent and slot labels using both real and synthetic data. We again utilize a pretrained transformer model and add two classification heads above the language model backbone. The classification heads and the backbone are jointly trained with cross entropy loss.

\paragraph{Noisy Labels}
Although we carefully tune the generator and filter the invalid samples to generate high quality data, it is inevitable to contain noise in the generated data. For example, some tokens might be labeled incorrectly in the synthetic data. The label noise can be detrimental since modern over-parametrized models can easily overfit to the noisy labels \cite{zhang2016understanding}. We rephrase learning with generated data as learning with noisy labels thus connect it with a well-studied literature \cite{song2020learning}. As a simple modification, we apply a recently proposed regularization, mixout \cite{lee2019mixout}. Mixout is originally proposed to improve the generalization of fine-tuning large-scale language models. It has been shown effective particularly when training data is limited. Here, we provide another perspective by showing that mixout regularization is naturally robust to label noise.

Mixout generalizes dropout by replacing the dropped parameters with pretrained parameters instead of zeros. We denote the mixture function between new parameters $\mathbf{w}$ and pretrained parameters $\mathbf{u}$ as $\Phi(\mathbf{w};\mathbf{u},\mathbf{M})$:
\begin{equation*}
    \Phi(\mathbf{w};\mathbf{u},\mathbf{M}) = \mu^{-1}(\mathbf{M}\mathbf{w}+(\mathbf{I}-\mathbf{M})\mathbf{u}-(1-\mu)\mathbf{u}).
\end{equation*}
$\mathbf{M}=\text{diag}(M_1,M_2,\ldots,M_d)$ represents the binary mask matrix, which is specified as random matrix during training and set to ones during testing. The mask matrix $\mathbf{M}$ satisfies $E(M_i)=\mu$ and $Var(M_i)=\sigma^2$ for all $i$. As proven in \cite{lee2019mixout} and repeated here in Theorem~\ref{th:mixout} for completeness, a strongly convex loss for mixout regularized parameters is lower bounded by a $L_2$ regularized loss.

\begin{theorem}\label{th:mixout}
Assume the loss function $\mathcal{L}$ is strongly convex, and the mask matrix $\mathbf{M}$ satisfies $E(M_i)=\mu$ and $Var(M_i)=\sigma^2$ for all $i$, then there exists $m > 0$ such that
\begin{equation*}
    \mathbb{E}\left[\mathcal{L}\left(\Phi(\mathbf{w};\mathbf{u},\mathbf{M})\right))\right] \geq \mathcal{L}(\mathbf{w}) + \frac{m\sigma^2}{2\mu^2}\Vert\mathbf{w}-\mathbf{u}\Vert
    _2^2.
\end{equation*}
The equality holds for quadratic loss functions. 
\end{theorem}

Therefore, optimizing a mixout regularized model results in minimizing a $L_2$ regularized loss when the loss function is strongly convex. Inspired by the reasoning in \cite{hu2019simple}, we show that mixout regularized model converges to the kernel ridge regression solution for sufficiently wide networks with MSE loss\footnote{Although cross entropy loss is used for experiments, MSE loss is simpler to derive theoretical results.}.

\begin{theorem}
\label{theor:mixout}
Assume the neural network is infinitely wide, the initialization satisfies $f(x;u) = 0$, then training the mixout regularized network $f(x;\Phi(\mathbf{w};\mathbf{u},\mathbf{M}))$ with MSE loss $\frac{1}{2}\Vert f(x;\Phi(\mathbf{w};\mathbf{u},\mathbf{M}))-y\Vert_2^2$ leads to the kernel ridge regression solution:
\begin{equation*}
    f^*(x;w)=k(x, \mathbf{X})^T(k(\mathbf{X}, \mathbf{X})+\lambda^2\mathbf{I})^{-1}\mathbf{Y}, \quad \lambda^2 = \frac{m\sigma^2}{\mu^2},
\end{equation*}
where $\mathbf{X}$ and $\mathbf{Y}$ are training data and labels. The kernel $k(\cdot,\cdot)$ is induced by the initial network and named Neural Tangent Kernel:
\begin{equation*}
    k(x,x') = \langle\phi(x), \phi(x')\rangle = \langle\nabla_\mathbf{\theta} f(x;\mathbf{u}),\nabla_\mathbf{\theta} f(x';\mathbf{u})\rangle.
\end{equation*}
\end{theorem}

The theorem also holds for sufficient wide network, and the initial condition $f(x;u)=0$ can be satisfied by simply manipulating the initialization of the classification heads. For example, initializing the heads as the difference of two same networks with the same initialization. Furthermore, \cite{hu2019simple} proves that the kernel ridge regression solution using neural tangent kernel trained with noisy data achieves a generalization bound on clean data close to the one achieved by training on clean data. Therefore, mixout regularization guarantees the generalization on clean data. In Appendix \ref{sec:proof} we provide a proof for Theorem \ref{theor:mixout}.

In summary, the mixout regularization plays two roles in our model: it regularizes the model so that it does not overfit the small training set; it also helps to resist label noise in generated data.

\subsection{Training Procedure}
As our model consists of two separate components, conditional generator and joint classifier, 
we would like to tune the generator so that the classifier trained with generated data can achieve the best performance on validation set. However, it would be difficult to determine the hyperparameters for generator based on validation set performance, since it requires to train a classifier till convergence. The likelihood of the validation set cannot be used as a metric, since some hyperparameters, like the probability of replacing a token, can affect the underlying conditional distribution $p(x,s,y \mid c)$, and comparing different conditional likelihood is meaningless. Instead, we propose two new metrics that measure the quality of the generation. We empirically verify the correlation between those metrics and the downstream task performance. Hence, we can search hyperparameters for the generator until it achieves the best metric scores without having to train the classifier.

Several metrics have been proposed to evaluate text generation in the past few years. BLEU \cite{papineni2002bleu} and Self-BLEU \cite{lu2018neural} are N-gram based metrics that measure quality and diversity respectively. Perplexity evaluate the likelihood of the generated text over a pretrained language model. \cite{zhao2018adversarially} train another language model over generated text and evaluate the likelihood for a held-out real dataset. Inspired by FID \cite{heusel2017gans}, \cite{semeniuta2018accurate} propose to compute the Frechet Distance over embedding from a pretrained InferSent model. \cite{zhang2019bertscore} leverage the BERT model to extract contextualized embedding and compute the similarity between real and generated text.

However, we notice that these metrics are all measuring the quality of the generated sentence; it does not consider the correspondence between tokens and their labels. For our case, the token-label correspondences are significant for downstream task performance. We propose two new metrics that take the correspondence between tokens and labels into consideration. Specifically, we extract sentence-level features from the augmented sentences and compute the Frechet distance \cite{heusel2017gans} and precision-recall scores \cite{sajjadi2018assessing} between real and generated sentences. The precision and recall scores measure how much of one distribution can be generated by another. Higher recall means a greater portion of samples from the true distribution are covered by the generator; and similarly, higher precision means a greater portion of samples from generator are covered by the real distribution. By using the augmented language format, we are essentially measuring the precision and recall between real and generated joint distributions $p(x,s,y)$. Note that if the joint distribution $p(x,s,y)$ matches, it implies the marginal distribution $p(x)$ also matches, but not vice versa. 

During experiment, we use T5 encoder \cite{raffel2019exploring} to extract contextualized embedding for tokens and average to get a sentence-level embedding. 
We denote the two metrics that use augmented language as input as \texttt{T5-FDa} and \texttt{T5-PRa}.
Since the pretrained T5 models are not trained with the augmented language format, we may desire a specialized T5 model to extract the embedding. We thus fine-tune the T5 model over the validation set for several iterations and denote the corresponding metrics \texttt{ft-T5-FDa} and \texttt{ft-T5-PRa} respectively.

\section{Related Works}
\subsection{Generative Data Augmentation} 
Data augmentation is a common strategy when faced with scarce training data. It has always been a key factor of performance improvement for computer vision applications \cite{shorten2019survey}. Certain transformations, such as shifting, cropping and flipping, can be applied to images, and they preserve the class information. However, augmentation for textual data is challenging. Simple transformations will distort the data and introduce grammar error. Word-level augmentations that replace a certain word with its synonym, delete a word, or change the word order have been shown effective for text classification \cite{wei2019eda}. Sentence-level augmentations, such as back-translation \cite{edunov2018understanding} and paraphrasing \cite{kumar2019submodular}, have also been used to improve downstream task performance. Recently, there has been a surge in using conditional generative models for data augmentation. \cite{antoniou2017data} use GANs to generate new images condition on an exemplar image. \cite{xia2020cg} employ VAE to generate utterence condition on intent labels. The most relevant work \cite{yoo2019data} uses a VAE model to jointly generate utterance and slot labels. However, their generative model does not leverage the pretrained language model and requires a large datasets to train.

\subsection{Noisy Labels}
Many methods have been proposed to learn from noisy labeled data. They can be roughly grouped into the following categories: \cite{vahdat2017toward,lee2018cleannet,li2017learning} rectify the noisy labels by utilizing an additional noise modeling module. \cite{sukhbaatar2014training,patrini2017making} correct the loss by estimating the probability of mislabeling one class to another, a.k.a. the label transition matrix. \cite{malach2017decoupling,jiang2018mentornet,han2018co,yu2019does} modify the training procedure to identify noisy labeled examples or to train with confident examples. \cite{ghosh2017robust,zhang2018generalized,wang2019symmetric,ma2020normalized} propose new loss functions that are robust to label noise. \cite{pereyra2017regularizing,hu2019simple} utilize regularizations to improve robustness of the model. In this work, we rephrase learning with generated data as learning with noisy labels and utilize the mixout regularization to improve robustness.

\section{Experiments}

\begin{table}[]
    \centering
    \caption{Dataset statistics}
    \label{tab:datasets}
    \resizebox{\linewidth}{!}{
    \begin{tabular}{cccccc}
    \toprule
              &  \#intent & \#slots & \#train & \#validation & \#test \\
    \midrule
    SNIPS     &  7       & 39     & 13084  & 700         & 700   \\
    ATIS      &  18      & 83     & 4478   & 500         & 893   \\
    NLUED     &  64      & 54     & 9960   & 1936        & 1076  \\
    \bottomrule
    \end{tabular}}
\end{table}

\begin{table*}[]
    \centering
    \small
    \caption{Slot labeling and intent classification performance with four different sampling ratios. Generator and classifier are fine-tuned T5-large and BERT-large respectively. Mixout with Bernoulli mask ($p=0.05$) is applied to the classifiers to regularize and resist noisy labels. We generate 500 synthetic data per intent for SNIPS and 50 per intent for ATIS and NLUED. Mean and standard deviation are reported using four different seeds.}
    \label{tab:results}
    \begin{subtable}[h]{\textwidth}
    \centering
    \caption{SNIPS}
    \resizebox{\textwidth}{!}{
    \begin{tabular}{ccccc|cccc}
    \toprule
          & \multicolumn{4}{c|}{Slot Labeling (F1)} & \multicolumn{4}{c}{Intent Classificaton (Acc.)} \\
          \cline{2-5}\cline{6-9}
          &  0.25\%  &  0.5\%  &  1\%  &  2\% & 0.25\%  &  0.5\%  &  1\%  &  2\% \\
    \midrule
    JBERT & 45.96\tpm5.16 &  62.22\tpm1.64 & 78.36\tpm1.86 & 88.94\tpm0.92 &  87.11\tpm8.09 & 90.96\tpm0.47 & 96.68\tpm0.40 & 98.21\tpm0.52 \\
    G(factor)+JBERT \cite{yoo2019data} & 48.34\tpm1.54 & 64.73\tpm2.15 & 79.10\tpm1.36 & 88.87\tpm1.21 & 88.91\tpm3.95 & 94.66\tpm0.92 & 96.36\tpm0.82 & 98.02\tpm0.57 \\
    G(intent)+JBERT [ours] & 52.26\tpm4.55 & 74.21\tpm3.62 & \textbf{85.33\tpm0.87} & 89.75\tpm0.45 & 91.68\tpm2.46 & 95.21\tpm2.29 & 97.46\tpm0.97 & 98.43\tpm0.65 \\
    G(words)+JBERT [ours] & 59.12\tpm4.27 & 73.30\tpm2.62 & 84.26\tpm0.64 & 89.86\tpm0.99 & 91.04\tpm1.59 & \textbf{96.50\tpm0.46} & 97.71\tpm0.79 & 98.46\tpm0.37 \\
    G(span)+JBERT [ours] & 61.24\tpm1.36 & \textbf{75.39\tpm2.76} & 84.99\tpm1.30 & 89.27\tpm0.30 & \textbf{93.21\tpm3.75} & 95.36\tpm3.00 & \textbf{97.79\tpm0.46} & \textbf{98.46\tpm0.23} \\
    G(multi\_spans)+JBERT [ours] & \textbf{63.00\tpm4.48} & 74.51\tpm1.67 & 85.03\tpm0.72 & \textbf{90.07\tpm0.91} & 91.68\tpm0.27 & 95.25\tpm0.97 & 97.57\tpm0.39 & 98.43\tpm0.30 \\
    \bottomrule
    \end{tabular}}
    \end{subtable}
    \begin{subtable}[h]{\textwidth}
    \centering
    \caption{ATIS}
    \resizebox{\textwidth}{!}{
    \begin{tabular}{ccccc|cccc}
    \toprule
         & \multicolumn{4}{c|}{Slot Labeling (F1)} & \multicolumn{4}{c}{Intent Classificaton (Acc.)} \\
          \cline{2-5}\cline{6-9}
         &  0.25\% & 0.5\% & 1\% & 2\% &  0.25\% & 0.5\% & 1\% & 2\% \\
    \midrule
    JBERT & 52.20\tpm2.83 & 63.25\tpm4.12 & 71.53\tpm2.52 & 79.67\tpm0.84 & 79.96\tpm1.64 & 83.76\tpm1.18 & 87.07\tpm2.21 & 90.12\tpm0.64 \\
    G(factor)+JBERT \cite{yoo2019data} & 53.68\tpm3.74 & 64.39\tpm3.71 & 72.68\tpm2.93 & 79.45\tpm0.90 & 77.99\tpm9.74 & \textbf{86.53\tpm3.31} & \textbf{91.09\tpm1.09} & \textbf{93.59\tpm1.65} \\
    G(intent)+JBERT [ours] & 55.52\tpm2.39 & 66.48\tpm2.49 & 74.39\tpm1.23 & 79.71\tpm1.67 & 70.46\tpm6.21 & 83.51\tpm5.19 & 87.29\tpm3.41 & 90.99\tpm0.88 \\
    G(words)+JBERT [ours] & \textbf{61.31\tpm1.66} & \textbf{69.49\tpm2.98} & \textbf{74.97\tpm0.73} & \textbf{81.79\tpm1.90} & \textbf{80.43\tpm3.86} & 84.71\tpm1.55 & 88.52\tpm2.45 & 92.86\tpm1.09 \\
    G(span)+JBERT [ours] & 58.08\tpm1.45 & 66.51\tpm2.38 & 72.68\tpm3.08 & 80.30\tpm1.72 & 76.26\tpm3.12 & 80.29\tpm5.39 & 84.71\tpm1.92 & 86.51\tpm2.60 \\
    G(multi\_spans)+JBERT [ours] & 58.21\tpm3.38 & 68.44\tpm2.41 & 73.63\tpm2.23 & 81.30\tpm1.04 & 73.26\tpm14.1 & 83.62\tpm2.45 & 88.27\tpm1.03 & 90.30\tpm3.23 \\
    \bottomrule
    \end{tabular}}
    \end{subtable}
    \begin{subtable}[h]{\textwidth}
    \centering
    \caption{NLUED}
    \resizebox{\textwidth}{!}{
    \begin{tabular}{ccccc|cccc}
    \toprule
         & \multicolumn{4}{c|}{Slot Labeling (F1)} & \multicolumn{4}{c}{Intent Classificaton (Acc.)} \\
          \cline{2-5}\cline{6-9}
         &  0.5\% & 1\% & 2\% & 4\% &  0.5\% & 1\% & 2\% & 4\% \\
    \midrule
    JBERT & 28.58\tpm2.17 & 39.21\tpm1.37 & 53.45\tpm1.66 & 60.93\tpm0.84 & 31.76\tpm1.54 & 53.42\tpm4.09 & 72.35\tpm0.52 & 79.44\tpm0.60 \\
    G(factor)+JBERT \cite{yoo2019data} & 30.34\tpm1.85 & 39.35\tpm0.95 & 52.07\tpm2.51 & 59.62\tpm1.10 & \textbf{59.87\tpm0.93} & \textbf{66.89\tpm1.24} & 74.65\tpm1.64 & \textbf{80.97\tpm0.75} \\
    G(intent)+JBERT [ours] & 32.39\tpm2.96 & 40.90\tpm1.90 & 50.71\tpm1.56 & 58.26\tpm1.03 & 45.24\tpm5.40 & 55.25\tpm3.25 & 72.21\tpm0.82 & 78.69\tpm1.00 \\
    G(words)+JBERT [ours] & 39.53\tpm1.15 & 42.94\tpm2.49 & 54.12\tpm1.54 & 60.07\tpm1.35 & 44.63\tpm4.47 & 58.99\tpm1.79 & 71.75\tpm1.78 & 79.69\tpm1.02 \\
    G(span)+JBERT [ours] & \textbf{42.00\tpm0.95} & \textbf{47.32\tpm1.61} & \textbf{56.11\tpm1.65} & \textbf{60.57\tpm1.08} & 51.97\tpm3.85 & 64.22\tpm0.96 & \textbf{74.86\tpm1.28} & 79.76\tpm0.43 \\
    G(multi\_spans)+JBERT [ours] & 38.94\tpm1.97 & 43.48\tpm0.93 & 52.36\tpm1.35& 59.35\tpm1.46 & 55.11\tpm2.12 & 60.27\tpm2.11 & 72.42\tpm2.75 & 79.86\tpm0.99 \\
    \bottomrule
    \end{tabular}}
    \end{subtable}
\end{table*}

\begin{table*}[]
    \centering
    \small
    \caption{Investigate the correlation between downstream task performance and different metrics. Slot labeling and intent classification performances are evaluated on SNIPS validation set. 100 synthetic data per intent are generated from each generator fine-tuned with 0.25\% of the training data. Numbers are averaged from 4 independent runs. \texttt{tau} is the correlation coefficient. Larger absolute value means stronger correlation.}
    \label{tab:metrics}
    \resizebox{\textwidth}{!}{
    \begin{tabular}{|| c | c |c c c c c | c c ||}
    \hline
    \multicolumn{2}{||c|}{} & factor & intent & words & span & multi\_spans & \multicolumn{2}{c||}{\multirow{2}{*}{tau}} \\
    \cline{1-7}
    \multicolumn{2}{||c|}{Slot labeling (F1) $\uparrow$} & 47.32 & 53.49 & 59.28 & 60.06 & 62.85 & & \\
    \cline{1-9}
    \multicolumn{2}{||c|}{Intent Classification (Acc.) $\uparrow$} & 88.07 & 90.64 & 90.68 & 90.68 & 90.32 & SL & IC \\
    \hline\hline
    \multirow{2}{*}{GPT2-Perplexity $\downarrow$} & real & \multicolumn{5}{c|}{381.46} & \multirow{2}{*}{0.20} & \multirow{2}{*}{0.53}  \\
    \cline{2-7}
                                     & fake & 158.43 & 271.90 & 327.78 & 297.59 & 306.38 & & \\
    \hline\hline
    BLEU-4 $\uparrow$ & real$\leftrightarrow$fake & 0.15 & 0.20 & 0.24 & 0.23 & 0.23 & 0.11 & 0.67 \\
    \hline
    \multirow{2}{*}{Self-BLEU-4 $\downarrow$} & real & \multicolumn{5}{c|}{0.28} & \multirow{2}{*}{0.20} & \multirow{2}{*}{0.53} \\
    \cline{2-7}
                                 & fake & 0.25 & 0.47 & 0.64 & 0.61 & 0.62 & & \\
    \hline\hline
    T5-FD $\downarrow$ & real$\leftrightarrow$fake & 2.150 & 2.627 & 2.201 & 2.180 & 2.284 & 0.20 & 0.11 \\
    \hline
    T5-PR $\uparrow$ & real$\leftrightarrow$fake & (85.4, 88.7) & (79.0, 82.9) & (82.5, 82.7) & (81.2, 85.1) & (80.0, 85.2) & (-0.40, 0.40) & (-0.11, -0.74) \\
    \hline\hline
    T5-FDa $\downarrow$ & real$\leftrightarrow$fake & 0.840 & 0.689 & 0.602 & 0.592 & 0.600 & -0.40 & -0.53 \\
    \hline
    T5-PRa $\uparrow$ & real$\leftrightarrow$fake & (81.7, 80.3) & (81.4, 76.2) & (83.3, 79.1) & (84.3, 79.7) & (84.0, 82.2) & (0.20, 0.40) & (0.32, -0.32) \\
    \hline\hline
    ft-T5-FDa $\downarrow$ & real$\leftrightarrow$fake & 0.490 & 0.361 & 0.317 & 0.306 & 0.306  & -0.53 & -0.44 \\
    \hline
    ft-T5-PRa $\uparrow$ & real$\leftrightarrow$fake & (82.3, 83.7) & (83.9, 76.2) & (84.8, 79.4) & (84.0, 84.7) & (84.2, 84.9) & (0.20, 0.60) & (0.53, -0.11) \\
    \hline
    \end{tabular}}
\end{table*}

In this section, we evaluate our framework on several public benchmark datasets, including SNIPS \cite{coucke2018snips}, ATIS \cite{hemphill1990atis} and NLUED \cite{liu2019benchmarking}. The dataset statistics are shown in Table~\ref{tab:datasets}. To simulate the limited data scenario, we sample a small subset at random from the original training set to train our models. Since the datasets are imbalanced, to make sure the sampled subset contains each intent class, we set a sampling ratio and sample each intent class based on the same ratio. The intent labels and slot labels are transformed to natural words separated by white spaces.

Baseline models include \textbf{JBERT} which directly fine-tune the BERT model with the small training set to jointly predict the intent and the slots. The classification heads are two fully connected layers. We also compare to a factorized augmentation strategy similar to \cite{yoo2019data} and denote it as \textbf{G(factor)+JBERT}, where the augmentation is performed by first sampling a sentence from a fine-tuned T5 model condition on the intent and then obtaining the slot labels from a fine-tuned BERT classifier with mixout regularization, i.e., $p(x,s,y)=p(y)p(x \mid y)p(s \mid x)$. Other classic data augmentation methods, such as EDA \cite{wei2019eda}, are not suitable for out setting since the augmentation may change the slot types. Our proposed methods are denoted as \textbf{G(intent)+JBERT}, \textbf{G(words)+JBERT}, \textbf{G(span)+JBERT} and \textbf{G(multi\_spans)+JBERT} to represent models with generator conditioning on intent labels and different masking schemes respectively. 

We fine-tune the \texttt{T5-large} as the generator and \texttt{BERT-large-cased} as the classifier for the main results. Ablation studies are performed in Sec.~\ref{sec:ablations}. To deal with the small training set size and noisy labels in generated data, mixout regularization is applied to the classifier in all models. During training, we sample each $M_i$ using a Bernoulli distribution with $p=0.05$. We generate 500 synthetic data per intent class for SNIPS dataset and 50 for ATIS and NLUED.

Table~\ref{tab:results} shows the slot labeling and intent classification performance on SNIPS, ATIS and NLUED. We conduct experiments with four different sampling ratios to subsample the training set. Intent classification performance is evaluated by accuracy and the slot labeling performance is evaluated by F1. Mean and standard deviation are reported from four independent runs.

The results (shown in Table~\ref{tab:results}) are consistent across different datasets. For slot labeling, the synthetic data can significantly improve the performance especially when training data is extremely low. For example, with only 0.25\% of the training data, we improve the slot labeling performance on SNIPS from 45.96 to 63.00, which is a 37.08\% relative improvement. Although all generation strategies are useful, conditioning on masked sentences are more effective than others with low training data. We believe that is because the masked sentences can provide more diverse context to synthesize the augmented sentences. It also provides a template so that the generator can leverage the token-label correspondences to generate faithful outputs. We can also see that using augmented language format gives better results than using the factorized generation scheme, which we believe is because of the difficulty of fine-tuning the classification head from scratch with only limited data. Given more real training data, the difference between different generation strategies gets smaller, which is as expected. For intent classification, the trends are similar, but the improvement is relatively small since the intent classification is a relatively simpler task. Using factorized generation is actually competitive, sometimes even better, for intent classification, because the synthetic sentences are generated directly conditioned on the intent labels. Fig~\ref{fig:generation} (Appendix) shows some generated sentences from our generator conditioned on augmented sentences masked with multiple spans. The generation is diverse and fluent. Thanks to the powerful pretrained models, our generator is capable of generating phrases beyond the training data.

\begin{figure*}
    \centering
    \begin{subfigure}[b]{0.24\textwidth}
        \centering
        \includegraphics[width=\textwidth]{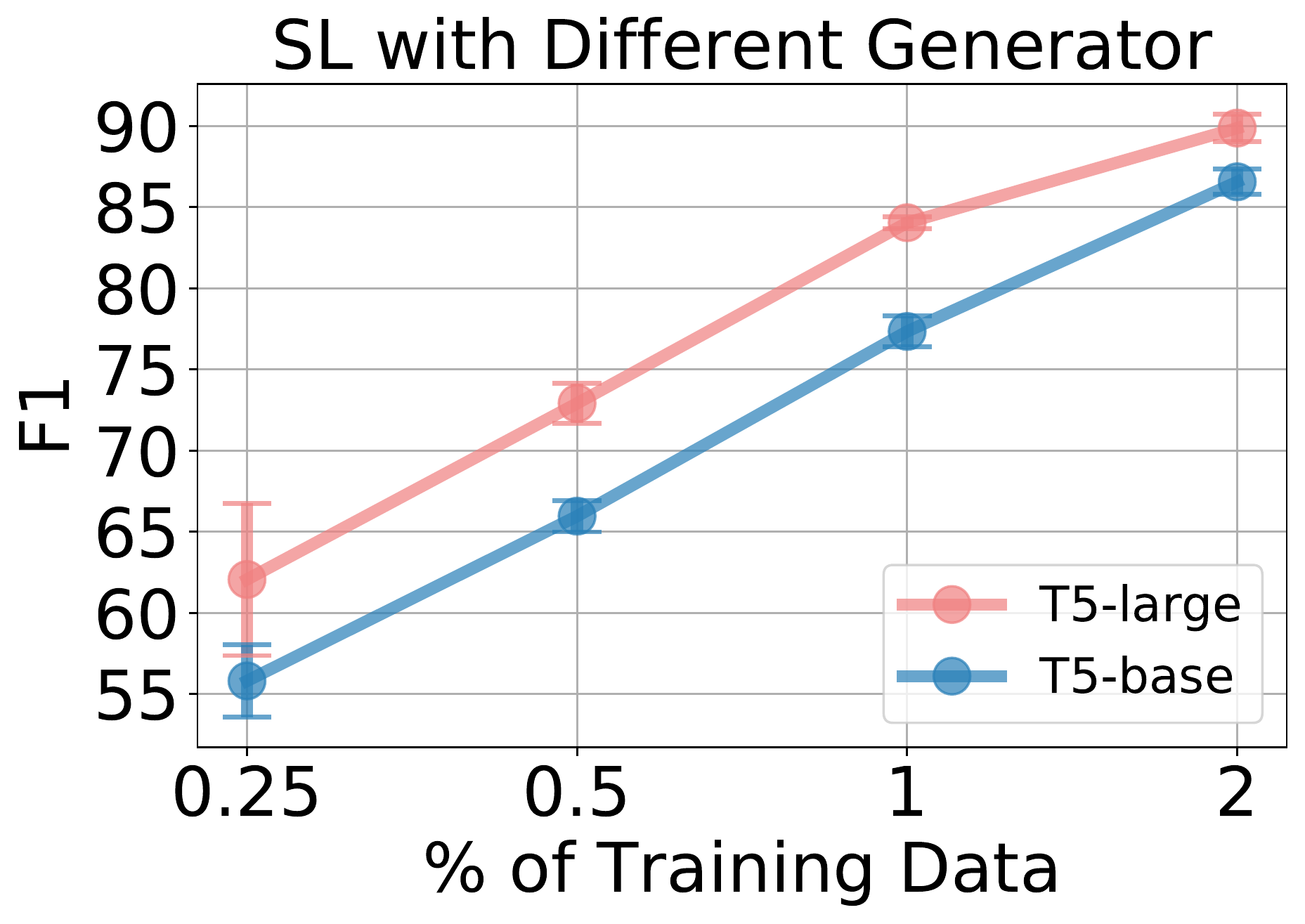}
        \caption{}
        \label{fig:ablation_generator_slots}
    \end{subfigure}
    \hfill
    \begin{subfigure}[b]{0.24\textwidth}
        \centering
        \includegraphics[width=\textwidth]{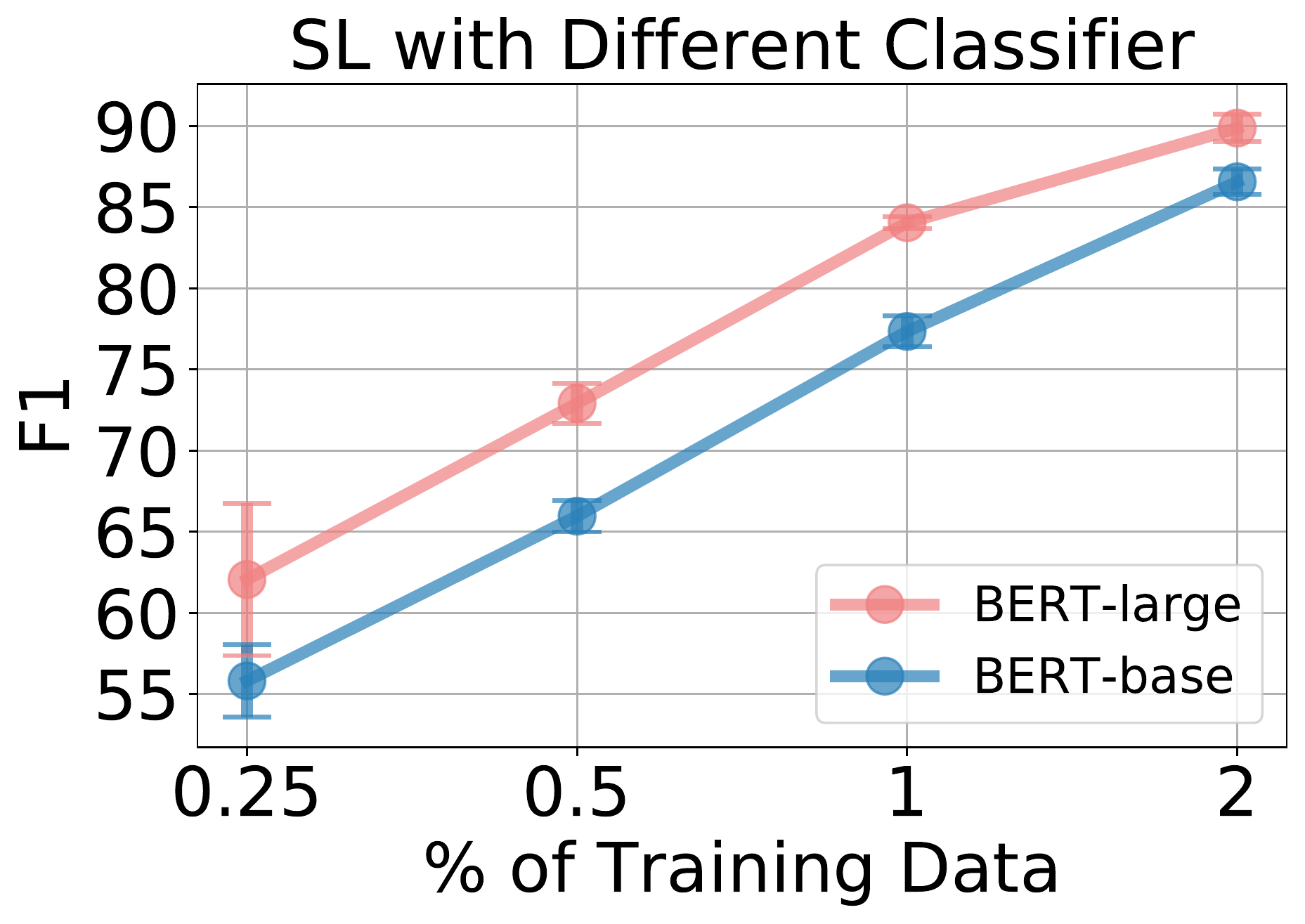}
        \caption{}
        \label{fig:ablation_classsifier_slots}
    \end{subfigure}
    \hfill
    \begin{subfigure}[b]{0.24\textwidth}
        \centering
        \includegraphics[width=\textwidth]{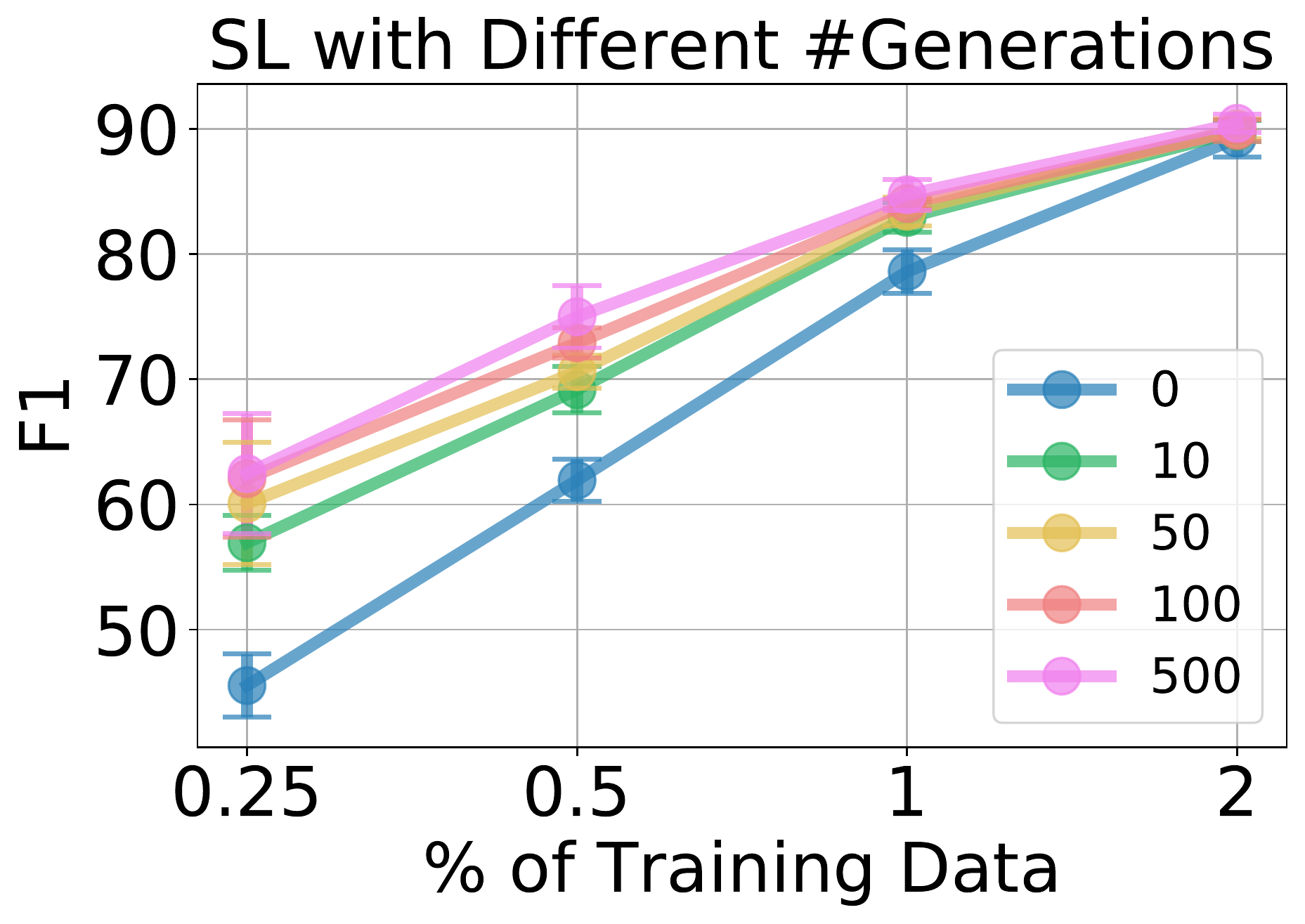}
        \caption{}
        \label{fig:ablation_generations_slots}
    \end{subfigure}
    \hfill
    \begin{subfigure}[b]{0.24\textwidth}
        \centering
        \includegraphics[width=\textwidth]{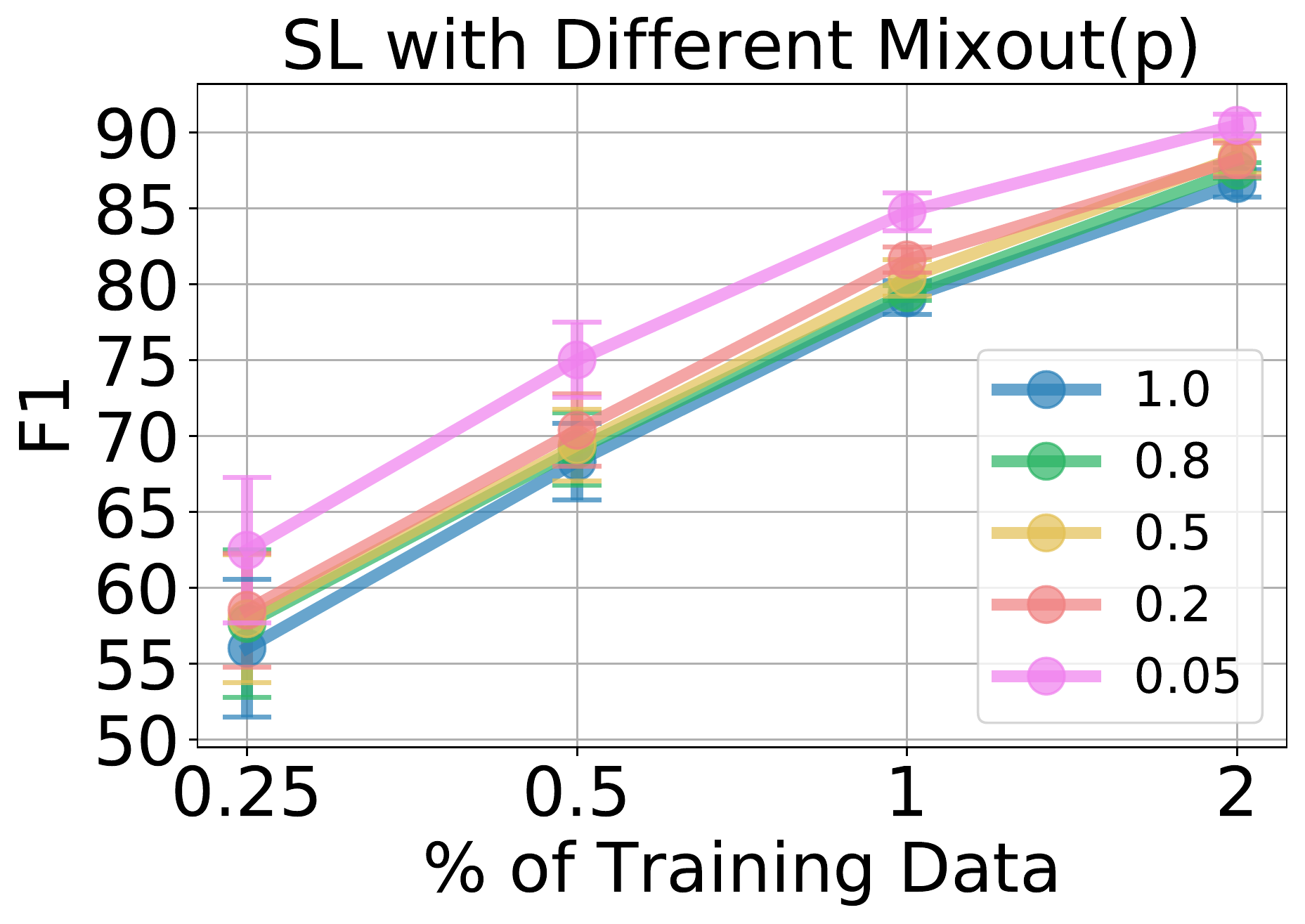}
        \caption{}
        \label{fig:ablation_mixout_slots}
    \end{subfigure}
    \begin{subfigure}[b]{0.24\textwidth}
        \centering
        \includegraphics[width=\textwidth]{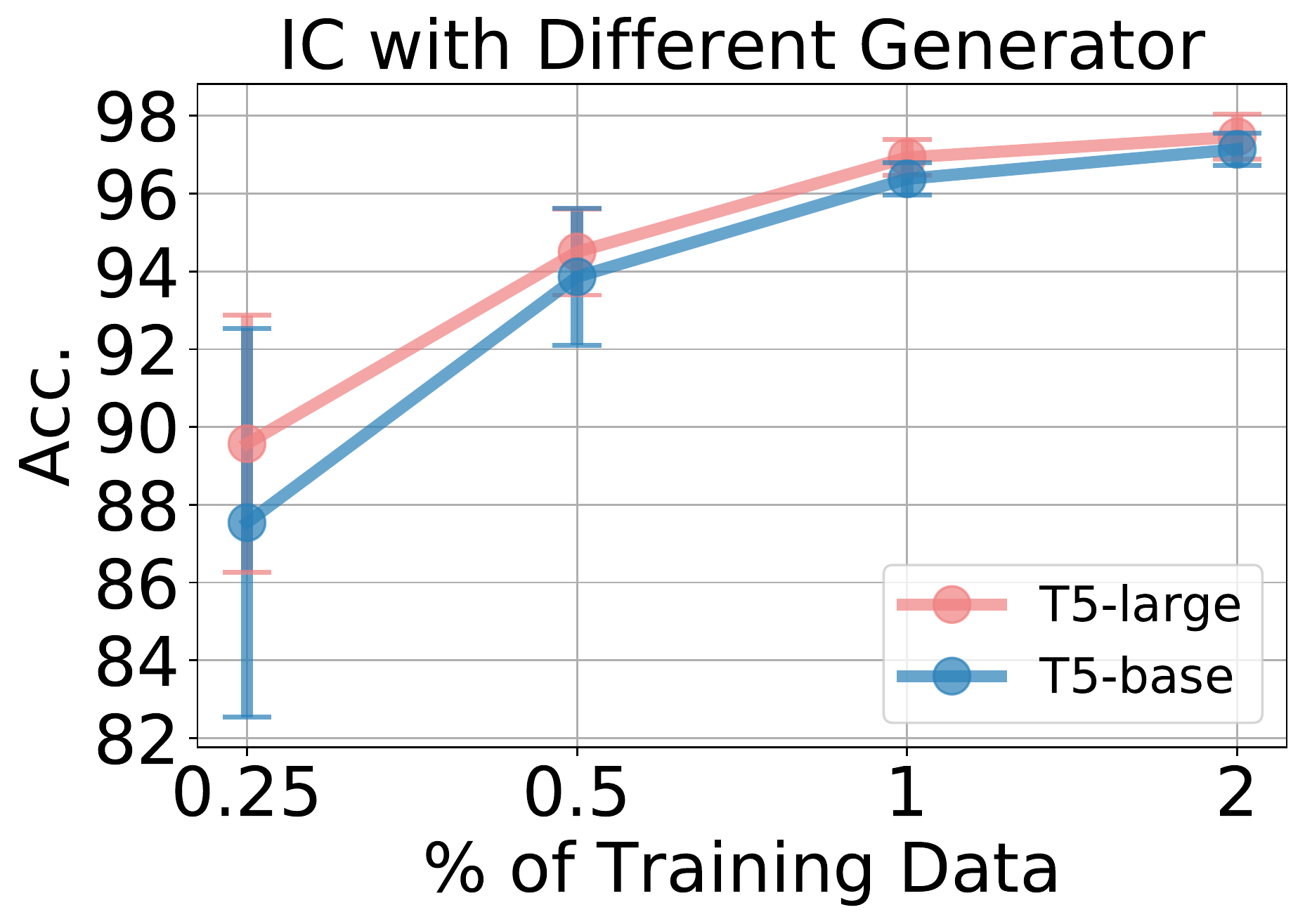}
        \caption{}
        \label{fig:ablation_generator_intent}
    \end{subfigure}
    \hfill
    \begin{subfigure}[b]{0.24\textwidth}
        \centering
        \includegraphics[width=\textwidth]{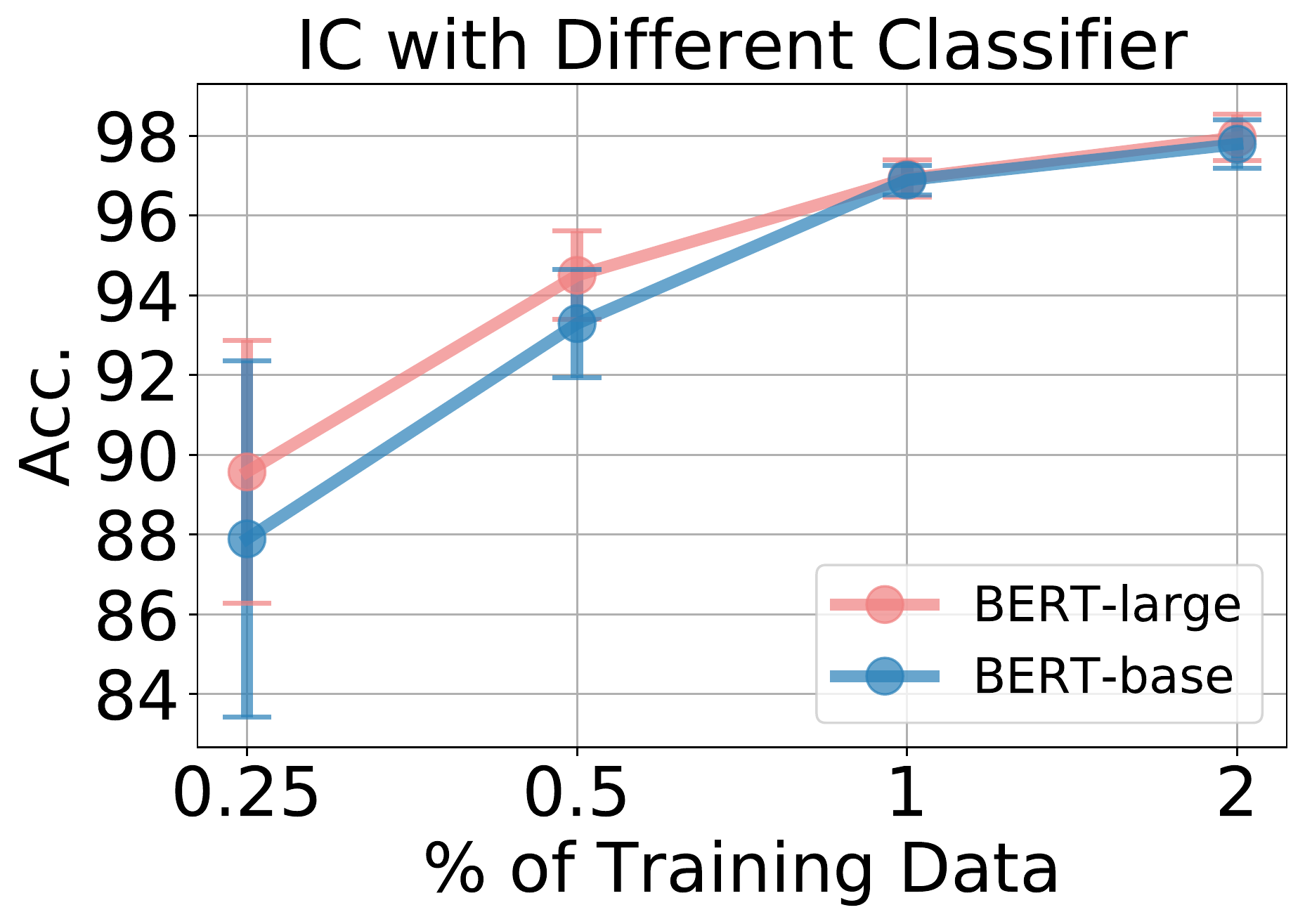}
        \caption{}
        \label{fig:ablation_classifier_intent}
    \end{subfigure}
    \hfill
    \begin{subfigure}[b]{0.24\textwidth}
        \centering
        \includegraphics[width=\textwidth]{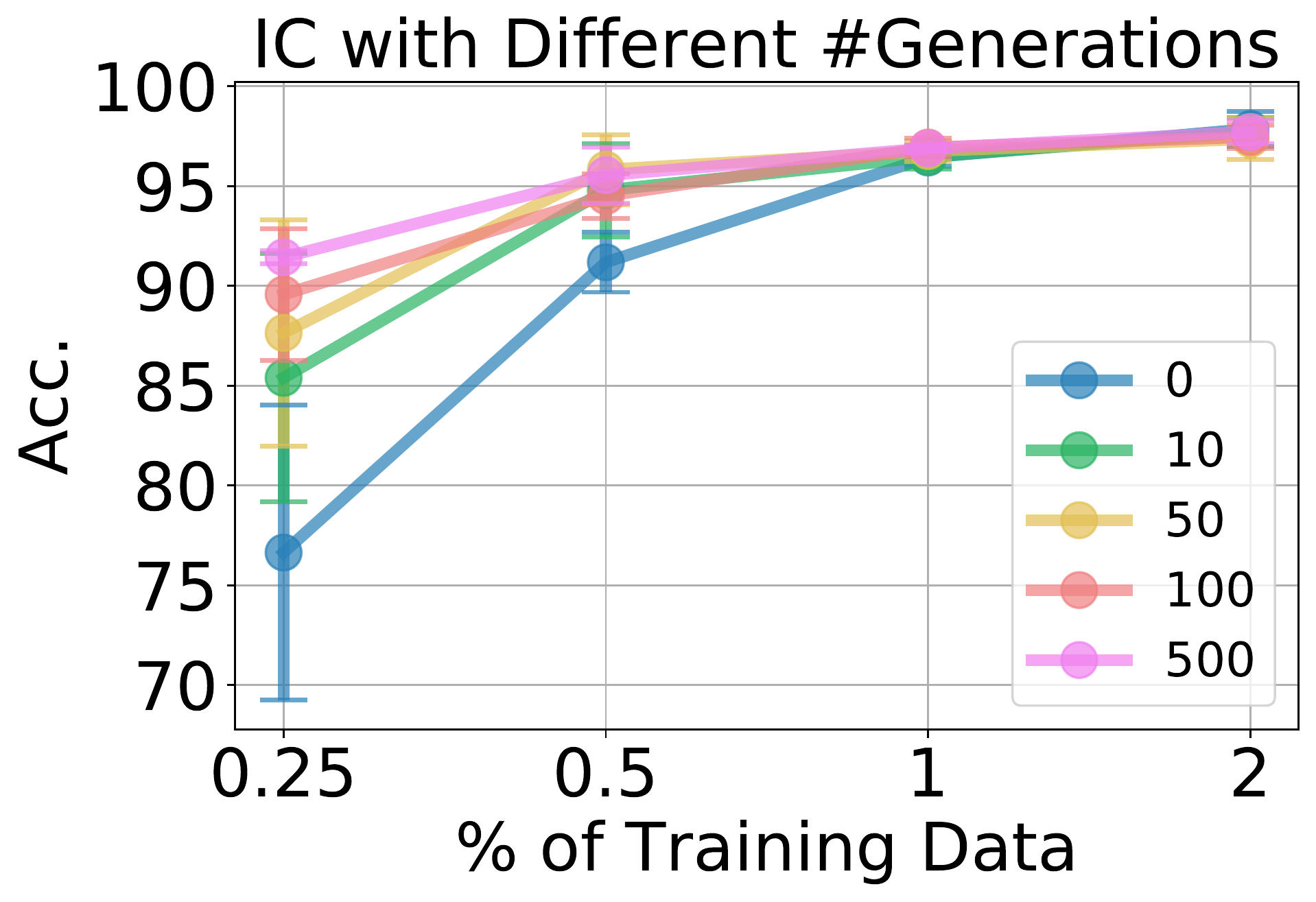}
        \caption{}
        \label{fig:ablation_generations_intent}
    \end{subfigure}
    \hfill
    \begin{subfigure}[b]{0.24\textwidth}
        \centering
        \includegraphics[width=\textwidth]{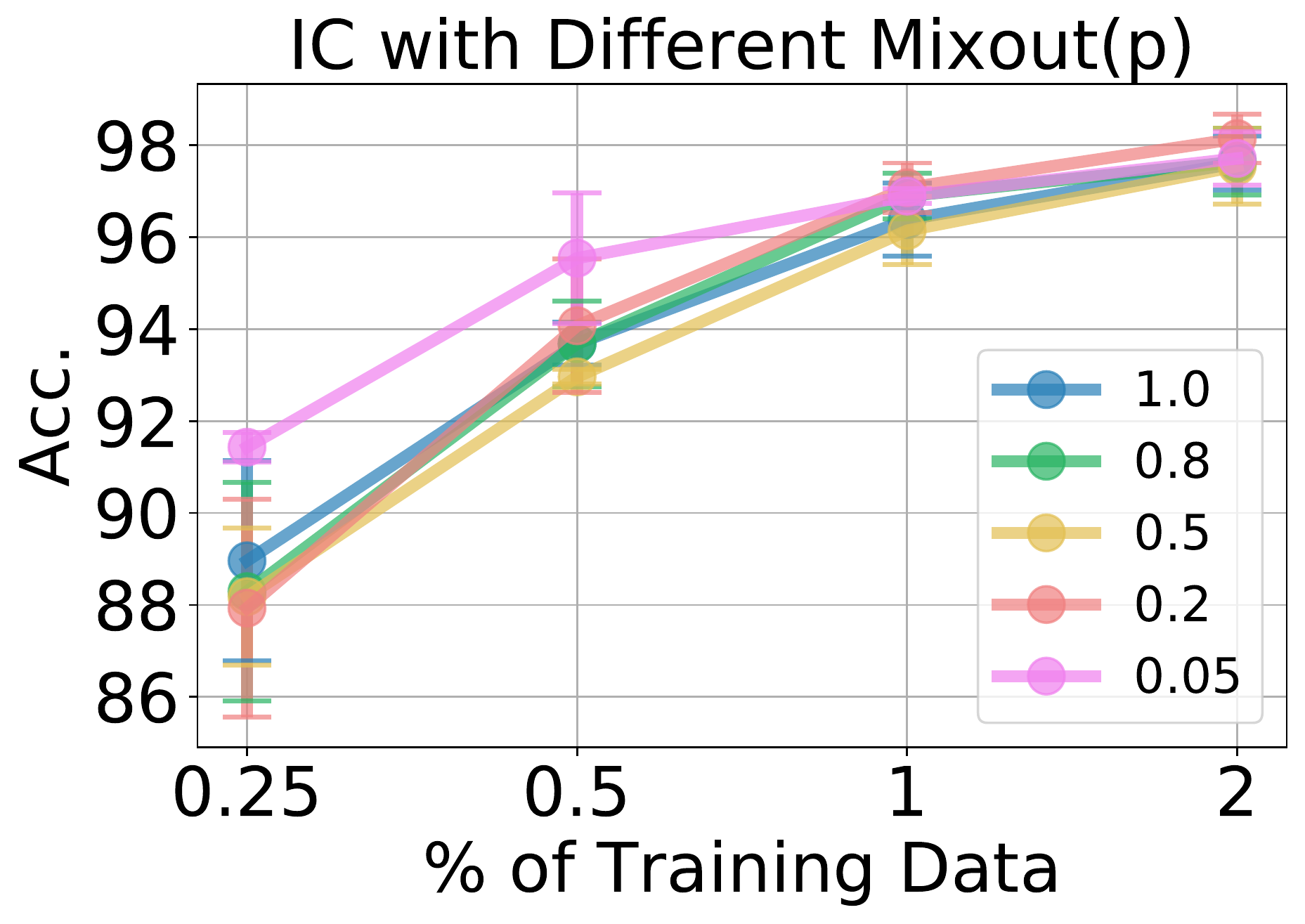}
        \caption{}
        \label{fig:ablation_mixout_intent}
    \end{subfigure}
    \caption{Ablation studies over different components of our framework. First row is for slot labeling, and the second row is for intent classification.}
    \label{fig:ablation}
\end{figure*}

\subsection{Metrics}
In Table~\ref{tab:metrics}, we empirically verify the correlation between several different types of metrics and the downstream task performance. The metrics are evaluated on SNIPS dataset and compare 100 real validation sentences per intent and 100 generated sentences per intent. We also report the slot labeling and intent classification performances on validation set using models trained with the corresponding generated data. The generators are \texttt{T5-large} fine-tuned with 0.25\% of the SNIPS training data using different generation schemes. The classifiers are \texttt{BERT-large-cased} fine-tuned with both small real training data and generated data. We use mixout with Bernoulli mask (p=0.05) for all classifiers. Numbers are averaged from four independent runs. 

The first metric we tested is the likelihood-based perplexity evaluated on a pretrained GPT2 model, which measures the fluency of the generated sentences. Then we tested the N-gram based BLEU and Self-BLEU scores, which quantify the quality and diversity of the generated sentences respectively. We also compute the Frechet distance and precision-recall scores from sentence embeddings extracted from a pretrained T5 encoder. Note that all the metrics mentioned above measure the quality of the generated sentences, i.e., the marginal distribution $p(x)$. We then utilize the augmented language format to evaluate T5-FDa and T5-PRa, which measure the joint distribution $p(x,s,y)$. Since T5 is trained with plain text, we fine-tune it with the augmented language format using the validation set for several iterations to calculate ft-T5-FDa and ft-T5-PRa. To quantitatively measure the correlation, we also calculate the Kendall's tau correlation coefficients for each metric with respect to the two downstream tasks.

From the quantitative measurement, we can see that measuring the marginal distribution cannot indicate the slot labeling performance, while sometimes it does show correlation to intent classification performance. For example, the factorized generation produces much better sentences based on the GPT2-Perplexity, but the slot labeling performance is actually worse than other models. N-gram based BLEU and Self-BLEU show better correlation to intent classification task, while they do not correlate very well with the slot labeling task. We think that is because the n-grams cannot capture the semantic meaning of tokens and they do not consider the correspondence between tokens and slot labels. The T5-FD and T5-PR metrics are counter-intuitive, since we expect T5-FD to have negative correlation and T5-PR to have positive correlation. Using augmented language format to compute these metrics (i.e., T5-FDa and T5-PRa) shows clearly better correlation to the slot labeling task. For precision-recall scores, the precision is better correlated to intent classification and recall is better correlated to slot labeling, which is inline with our intuition that intent classification benefits more from fluent utterances and slot labeling benefits more from diverse tokens.

\subsection{Ablations}\label{sec:ablations}

In this section, we perform ablation studies on different components of our framework. We conduct experiments on SNIPS dataset with generators conditioned on masked sentences. Multiple spans are masked out for conditioning. Fig.~\ref{fig:ablation} shows the results of varying each component respectively. From Fig.~\ref{fig:ablation_generator_slots} and \ref{fig:ablation_generator_intent}, we can see that larger generator leads to better performance, especially for slot labeling. Larger classifier can also improve the performance as shown in Fig.~\ref{fig:ablation_classsifier_slots} and \ref{fig:ablation_classifier_intent}. Fig.~\ref{fig:ablation_generations_slots} and \ref{fig:ablation_generations_intent} show that generating more synthetic data can improve the performance, especially for low data regime. With more real data given, however, the synthetic data influence the final performance less. Mixout regularization can significantly improve the performance as shown in Fig.~\ref{fig:ablation_mixout_slots} and \ref{fig:ablation_mixout_intent}. Stronger regularization typically results in better performance.

\section{Conclusion}
In this work, we present a framework to extract prior information from pretrained language models to improve spoken language understanding. We express the prior knowledge in the form of synthetic data and propose to use an augmented language format to generate both sentences and intent and slot labels simultaneously. The generated data as well as the small real dataset are used to fine-tune a classifier to predict intent and slot labels jointly. We also utilize the mixout regularization for classifiers to resist label noise in generated data. On three public benchmark datasets, we achieve superior performance over baselines.

For future directions, we will apply the augmented language format for other tasks, such as named entity recognition. We will also explore the classical few-shot setting, where we have other related tasks to accumulate common knowledge. 

\bibliographystyle{unsrt}
\bibliography{main}

\clearpage
\appendix
\setcounter{figure}{0}
\setcounter{table}{0}
\setcounter{footnote}{0}
\renewcommand\thefigure{\thesection.\arabic{figure}}
\renewcommand\thetable{\thesection.\arabic{table}}

\section{Proof}
\label{sec:proof}
\setcounter{theorem}{1}
\begin{theorem}
Assume the neural network is infinitely wide, the initialization satisfies $f(x;u) = 0$, then training the mixout regularized network $f(x;\Phi(\mathbf{w};\mathbf{u},\mathbf{M}))$ with MSE loss $\frac{1}{2}\Vert f(x;\Phi(\mathbf{w};\mathbf{u},\mathbf{M}))-y\Vert_2^2$ leads to the kernel ridge regression solution:
\begin{equation*}
    f^*(x;w)=k(x, \mathbf{X})^T(k(\mathbf{X}, \mathbf{X})+\lambda^2\mathbf{I})^{-1}\mathbf{Y}, \quad \lambda^2 = \frac{m\sigma^2}{\mu^2},
\end{equation*}
where $\mathbf{X}$ and $\mathbf{Y}$ are training data and labels. The kernel $k(\cdot,\cdot)$ is induced by the initial network and named Neural Tangent Kernel:
\begin{equation*}
    k(x,x') = \langle\phi(x), \phi(x')\rangle = \langle\nabla_\mathbf{\theta} f(x;\mathbf{u}),\nabla_\mathbf{\theta} f(x';\mathbf{u})\rangle.
\end{equation*}
\end{theorem}

\begin{proof}
According to the neural tangent kernel theorem \cite{jacot2018neural,arora2019exact}, infinitely wide network $f(\cdot;\theta)$ is equivalent to its first-order approximation:
\begin{equation*}
    f(x;\theta) = f(x;u) + \langle \nabla_\theta f(x;u),\theta-u\rangle.
\end{equation*}
The approximation is also accurate for sufficiently wide network.
Since the initialization $f(x;u)=0$, if we denote features as $\phi(x) = \nabla_\theta f(x;u)$, then
\begin{equation*}
    f(x;\theta)=\phi(x)^T(\theta-u).
\end{equation*}
Based on Theorem 1, the MSE loss (a quadratic loss) on mixout regularized parameters is equivalent to a $L_2$ regularized loss, i.e.,
\begin{align*}
    \mathcal{L} &= \frac{1}{2}\Vert f(x;w) - y \Vert_2^2 + \frac{m \sigma^2}{2\mu^2}\Vert w - u \Vert_2^2 \\
    &= \frac{1}{2}\Vert \phi(x)^T(w-u) - y \Vert_2^2 + \frac{m \sigma^2}{2\mu^2}\Vert w - u \Vert_2^2.
\end{align*}
The loss on the whole training set $\mathbf{X}$ and $\mathbf{Y}$ is 
\begin{equation*}
    \mathcal{L} = \frac{1}{2}\Vert \phi(\mathbf{X})^T(w-u)-\mathbf{Y} \Vert_2^2 + \frac{m \sigma^2}{2 \mu^2} \Vert w-u \Vert_2^2.
\end{equation*}
By setting $\frac{\partial \mathcal{L}}{\partial w} = 0$, we have
\begin{equation*}
    w-u = \phi(\mathbf{X})(k(\mathbf{X},\mathbf{X}) + \lambda^2 \mathbf{I})^{-1}\mathbf{Y}.
\end{equation*}
Therefore, 
\begin{align*}
    f^*(x;w) &= \phi(x)^T(w-u) \\
    &= k(x, \mathbf{X})^T(k(\mathbf{X}, \mathbf{X})+\lambda^2\mathbf{I})^{-1}\mathbf{Y}.
\end{align*}
\end{proof}

\section{Augmented Language Format}

Figure~\ref{fig:aug_lan} shows how to convert between the classic BIO labeled data and our proposed augmented language format. We use special tokens to separate the utterances and their corresponding labels. The span of each slot type is explicitly marked with brackets. All labels are converted to natural words so that the pretrained tokenizer can directly process them. The conversion is bijective, therefore we can transform the augmented sentence bask to BIO format without loss of any information.

\begin{figure*}
    \centering
    \includegraphics[width=0.8\linewidth]{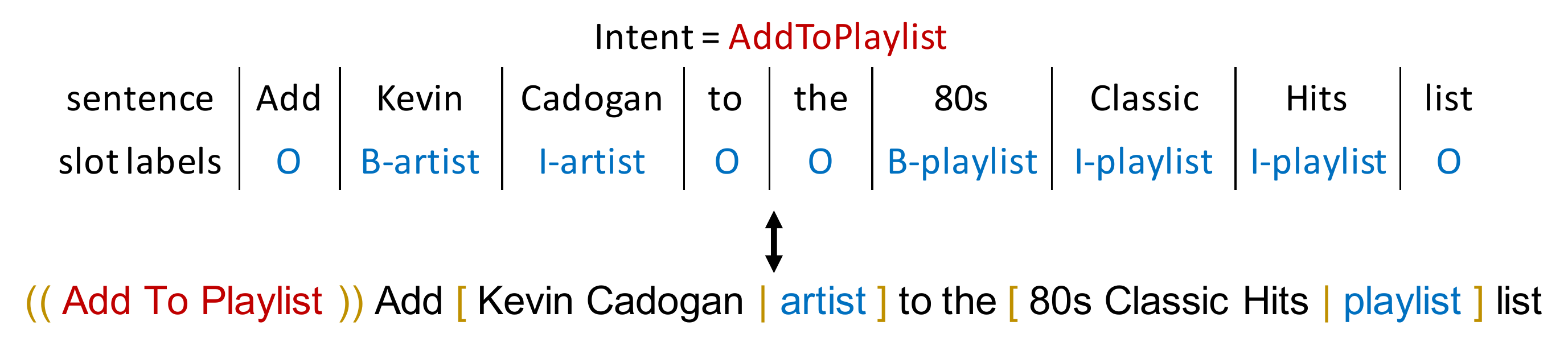}
    \caption{Convert a BIO labeled sentence to the augmented format.}
    \label{fig:aug_lan}
\end{figure*}

\section{Generation}

Figure~\ref{fig:generation} shows some samples from our generator for SNIPS dataset. The generations are diverse and fluent. The colored phrases are novel generations not presented in the training data. We can see that our generator is capable of drawing innovative tokens corresponding to their slot labels. For example, the generated artist, movie and city are all valid names. The superior generation ability comes from the pretrained language model. Our generation strategy essentially extracts informative prior knowledge for the pretrained language models.

\begin{figure*}
    \centering
    \includegraphics[width=0.98\linewidth]{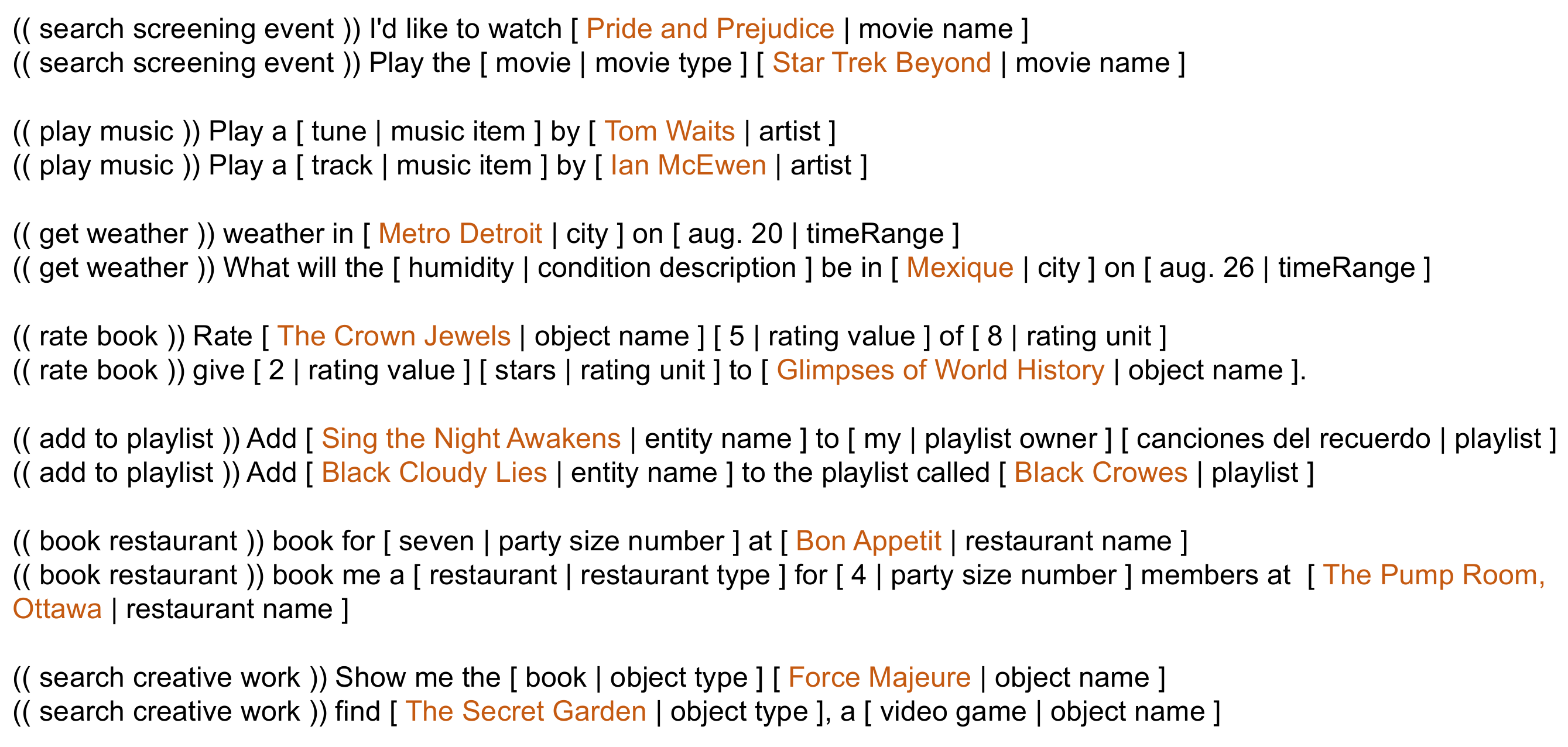}
    \caption{Some examples of the generated data. The generator is conditioned on the augmented sentences masked with multiple spans. The colored tokens are novel phrases not presented in the training data.}
    \label{fig:generation}
\end{figure*}

\end{document}